\documentclass[runningheads]{llncs}
\usepackage{geometry}
\usepackage[T1]{fontenc}
\usepackage{graphicx}
\usepackage{amsmath}
\usepackage{amssymb}
\usepackage{esvect}
\usepackage{bm}
\usepackage{url}
\usepackage{hyperref}
\usepackage{xcolor}
\usepackage{orcidlink}

\renewcommand{\orcidID}[1]{\textcolor{green!50!black}{\orcidlink{#1}}}

\begin{document}

\title{Designing Sustainable Federated Learning as a Service using Neural Architecture Search}

\titlerunning{Sustainable FLaaS using NAS}

\author{
Keya Patel\inst{1}\orcidID{0000-0003-3987-9828}\and
Sajib Mistry\inst{1}\orcidID{0000-0001-7513-3789}\and
Sheik Fattah\inst{1}\orcidID{0000-0002-9103-6089}\and
Deepak Kanneganti\inst{1}\orcidID{0009-0007-5860-2644}\and
Aneesh Krishna\inst{1}\orcidID{0000-0001-8637-5732}\and
Mufti Mahmud\inst{2}\and
Monowar Bhuyan\inst{3}\orcidID{0000-0002-9842-7840}
}

\authorrunning{K. Patel et al.}

\institute{
School of Electrical Engineering, Computing and Mathematical Sciences\\
Curtin University, Perth, Australia\\
\email{k.patel38@postgrad.curtin.edu.au}\\
\email{\{sajib.mistry, sheik.fattah, s.kanneganti, a.krishna\}@curtin.edu.au}
\and
King Fahd University of Petroleum and Minerals, Saudi Arabia\\
\email{mufti.mahmud@kfupm.edu.sa}
\and
Ume\r{a} University, Ume\r{a}, SE-90187, Sweden\\
\email{monowar@cs.umu.se}
\vspace{-7mm}
}

\maketitle

\begin{abstract}
The sustainability constraints of FLaaS consumers pose significant challenges to maintaining carbon-feasible federated training in FLaaS environments. These constraints often lead to infeasible consumer participation and unstable federated training under hard carbon constraints. We propose a \textit{Sustainable Federated Learning as a Service (SFLaaS)}, a carbon-constrained \textit{Neural Architecture Search (NAS)} framework for heterogeneous sustainable constraints. We introduce a requirement-driven search space that transforms consumer sustainability profiles into a feasible architecture region before federated execution. We develop a consumer-level carbon feasibility estimation mechanism to evaluate candidate architectures under dynamic carbon conditions. We propose a sustainable consumer scheduling strategy that adaptively selects feasible consumers and allocates local workloads to preserve consumer participation and statistical data coverage. An evolutionary search strategy jointly optimised for predictive performance, consumer feasibility, and participation coverage under hard carbon constraints. Experiments on real-world datasets and a simulated environment demonstrate the effectiveness of the proposed approach.

\keywords{Federated Learning as a Service  \and Green AI \and Neural Architecture Search \and Sustainable ML \and Carbon-aware Computing.}
\end{abstract}
\section{Introduction}
Machine learning (ML) has been increasingly adopted in domains such as healthcare, finance, and smart cities to enable intelligent applications and data-driven decision-making. However, centralised data collection and processing often introduce privacy and security concerns \cite{wen2023survey}. Federated Learning (FL) addresses this challenge by allowing multiple participants to train a shared model without moving raw data off-site. Each participant performs local training on its own data and sends model updates to a coordinating server, which aggregates them into a global model over multiple communication rounds \cite{gao2024flaas}.

FL is increasingly being offered as a cloud-based service, giving rise to Federated Learning as a Service (FLaaS) \cite{kourtellis2020flaas}. In FLaaS, the provider manages federated training across distributed organisations. As illustrated in Fig.~\ref{fig:Generic FLaaS}, the provider manages the overall learning process, including service orchestration, resource management, consumer scheduling, workload allocation, and model aggregation. Service consumers participate in FLaaS training using their private datasets to gain knowledge and collaboratively improve prediction performance.
\begin{figure}
\vspace{-5mm}
\includegraphics[width=\textwidth]{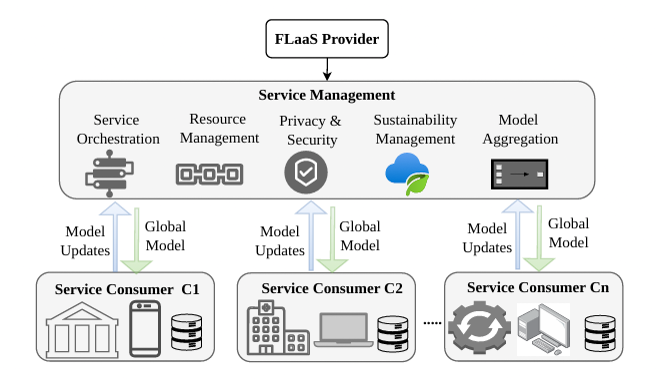}
\caption{An FLaaS System Architecture} 
\label{fig:Generic FLaaS}
\vspace{-5mm}
\end{figure}

FLaaS enables consumers to adopt FL without centrally collecting sensitive data from distributed environments. However, consumers in FLaaS often operate under different sustainability constraints, such as carbon budgets that limit their participation in federated training \cite{albelaihi2022green}. Such constraints are important because participating organisations may operate under energy budgets, carbon-emission policies, and green computing objectives aimed at reducing operational costs and satisfying sustainability requirements. Since federated training requires repeated local computation and frequent communication, participation in FLaaS significantly increases consumer-side energy consumption and carbon emissions. Sustainability constraints are influenced by factors such as computational capabilities, local dataset size, time-varying carbon emissions, and carbon budgets. Consequently, maintaining feasible and sustainable consumer participation creates a significant challenge in the FLaaS environment. Therefore, our focus is to design a \textit{sustainable FLaaS} that enables consumers to participate effectively under sustainability constraints.

The model architecture plays a critical role in determining consumer feasibility, which refers to the ability of consumers to participate in FL training while satisfying their sustainability constraints. {\textit{Neural Architecture Search (NAS)}} provides an opportunity to automatically design model architectures \cite{salmani2025systematic}. By exploring architectures with varying computational and communication costs, NAS identifies architectures that achieve high predictive performance while ensuring consumer feasibility within sustainability constraints.

Existing Green FL studies reduce energy and carbon emissions using techniques such as client selection, compression, and carbon-aware scheduling \cite{albelaihi2022green}, \cite{li2021fedgreen}, \cite{arputharaj2025green}. Similarly, Green and carbon-aware NAS methods optimise architecture design to reduce model complexity, communication costs, and energy consumption \cite{xu2021knas}, \cite{zhou2020econas}. However, these approaches neither consider hard carbon budgets nor ensure architectural feasibility for sustainability constraints. To the best of our knowledge, sustainability constraints remain largely unexplored in both FLaaS design and in neural architecture optimisation. Therefore, \textit{we investigate how NAS can be leveraged to design a sustainable FLaaS service that enhances consumer participation while maintaining learning quality}. We identify three \textit{key challenges} in applying NAS to design a sustainable FLaaS.

The primary challenge is \textit{service architecture design}, where an architecture remains feasible across consumers. Larger architectures require more computation per local epoch, increase the size of model updates, and therefore raise the carbon cost imposed on consumers. Consequently, an architecture that is feasible for high-resource consumers may become infeasible for consumers under stricter sustainability constraints. \textit{Consumer feasibility} under time-varying conditions is another challenge. Repeated local training gradually depletes consumer carbon budgets, while time-varying grid carbon intensity changes the carbon cost of participation. As a result, consumers feasible in early communication rounds may become infeasible later, leading to consumer dropout and reduced service continuity. The third challenge is maintaining a \textit{participation coverage} under dynamic sustainability conditions. For example, consumers in similar geographic regions may simultaneously experience high grid carbon intensity, making participation in the same round infeasible. Such correlated dropouts reduce the diversity of training data and increase workload imbalance, negatively affecting model convergence and generalisation. To address these challenges, we formulate NAS as a sustainability optimisation problem in an FLaaS environment. The key contributions of this paper are as follows.
\vspace{-2mm}
\begin{itemize}
   \item We design a requirement-driven macro NAS framework that integrates consumer sustainability constraints into architecture feasibility estimation.
    \item We develop a \textit{Consumer-Level Carbon Feasibility Estimation} mechanism that determines whether consumers can participate under candidate architectures and workload configurations while satisfying carbon constraints.
    \item We propose a \textit{Sustainable Consumer Scheduling} mechanism that preserves consumer participation under dynamic carbon constraints.
    \item We introduce an \textit{Evolutionary Search Strategy} that optimises performance, individual consumer feasibility, and participation coverage in FLaaS.
\end{itemize}

\section{Related Work}
FLaaS extends traditional FL by providing model orchestration, aggregation, and lifecycle management as cloud-based services. Early FLaaS enables collaborative model training across third-party applications while preserving privacy and enabling scalable service deployment \cite{gao2024flaas}, \cite{kourtellis2020flaas}. TruFLaaS further enhanced service reliability through trust management and secure aggregation for industrial IoT applications \cite{mazzocca2023truflaas}. More recently, hierarchical FLaaS frameworks have been proposed for edge-cloud environments to support heterogeneous models across resource-constrained IoT devices and cloud servers \cite{gao2024flaas}. NebulaFL further extended FLaaS to decentralised multi-cloud systems using asynchronous training and resource-aware scheduling to improve scalability and reduce communication overhead \cite{gao2024nebulafl}. However, these studies mainly focus on orchestration and scalability rather than sustainable optimisation under consumer-level carbon constraints. Green FL aims to reduce the environmental cost of distributed training through energy-efficient scheduling and communication optimisation. FedCS introduced a client selection strategy based on computational capability and communication latency \cite{nishio2019client}, while Oort improved participant selection using statistical utility and runtime performance \cite{lai2021oort}. GREED later proposed an energy-aware client selection mechanism that explicitly considers energy consumption and communication cost during training \cite{albelaihi2022green}. FedGreen additionally reduced communication overhead through fine-grained gradient compression for mobile edge environments \cite{li2021fedgreen}. Although these methods improve training efficiency, they primarily optimise execution-level behaviour, leaving architecture-level sustainability largely unexplored.

NAS has evolved towards hardware-aware and energy-aware optimisation by incorporating latency and energy efficiency into architecture design. DARTS introduced differentiable architecture search for efficient gradient-based optimisation \cite{liu2019darts}, while ProxylessNAS, MnasNet, and FBNet integrated hardware-aware latency optimisation for mobile and embedded platforms \cite{cai2019proxylessnas}, \cite{tan2019mnasnet},\cite{wu2019fbnet}. More recently, KNAS and EA-HAS-Bench demonstrated that architecture design and hyperparameter selection significantly affect training energy consumption \cite{xu2021knas}, \cite{dou2023ea}. CarbonTracker and EcoNAS further incorporated carbon and energy measurements into model optimisation and architecture search \cite{anthony2020carbontracker},\cite{zhou2020econas}. However, these approaches are designed for centralised infrastructures and do not consider geographically distributed participants, diverse hardware capabilities, or consumer-level carbon constraints. FedNAS integrated differentiable NAS into FL to optimise architectures across distributed consumers \cite{he2021fednas}. Nevertheless, existing federated NAS approaches primarily focus on predictive accuracy and communication efficiency rather than sustainable service orchestration. Consequently, limited work considers sustainable architecture search, consumer feasibility, and carbon constraints in FLaaS.

\section{Motivation Scenario}
Let us consider an FLaaS provider delivering a cancer detection service across geographically distributed hospital networks in Sydney ($c_1$), Singapore ($c_2$), Berlin ($c_3$), and Toronto ($c_4$). Each hospital has private medical imaging data collected from local patients. However, hospitals may not have sufficient data diversity to train an accurate cancer detection model. Therefore, hospitals participate in the FLaaS to gain knowledge from distributed medical data while preserving patient privacy. In cancer detection services, insufficient predictive accuracy may reduce diagnostic reliability and affect clinical decision-making. Therefore, the provider is required to achieve at least 85\% classification accuracy within 20 communication rounds while maintaining feasible consumer participation. However, the participating hospitals operate under sustainability constraints as shown in Table~\ref{tab:consumer_profiles}.

Designing model architectures that remain feasible across heterogeneous consumers is challenging. Let us consider the cancer detection service, where the provider evaluates large architecture, $\alpha_L$ with $\text{Cost}(\alpha_L)=4.1$ GFLOPs/epoch and $|\alpha_L|=25$M parameters, achieving 87\% accuracy. The small architecture, $\alpha_S$ with $\text{Cost}(\alpha_S)=0.6$ GFLOPs/epoch and $|\alpha_S|=4$M parameters, achieving 83\% accuracy. Although $\alpha_L$ provides higher accuracy, its high computational and communication costs may make it infeasible for hospitals under hard carbon budgets.

Maintaining consumer feasibility under dynamic conditions is another key challenge. Suppose hospital $c_2$ in Singapore is constrained to a maximum carbon emission of $0.70$ gCO$_2$. When $\alpha_L$ is trained with five local epochs, the resulting carbon cost for $c_2$ becomes approximately $0.84$ gCO$_2$, exceeding the allowable limit. Consequently, hospital $c_2$, which initially has sufficient carbon budget to participate, later becomes infeasible. The FLaaS provider may reduce the local epoch assignment for $c_2$ to maintain participation. However, repeated workload adjustments across communication rounds can create training imbalance.

Another challenge is in maintaining participation coverage across consumers. Suppose multiple hospitals in similar geographic regions simultaneously experience periods of high electricity-related carbon emissions. Under architecture $\alpha_L$, several hospitals may become infeasible within the same communication round. The FLaaS provider may exclude infeasible consumers, such as hospital $c_2$, to maintain training continuity. However, excluding $c_2$ removes 8,000 medical images from a distinct patient subpopulation, thereby reducing the diversity of the training data and negatively affecting the performance of FLaaS. Therefore, the FLaaS provider requires a sustainable NAS strategy that balances model performance, consumer feasibility, and stable participation under sustainability constraints.
\begin{table}[t]
\centering
\caption{Consumer constraints in the motivating example.}
\label{tab:consumer_profiles}
\begin{tabular}{lcccc}
\hline
\textbf{Consumer} & $\boldsymbol{\eta_i}$ (GFLOPs/J) & $\boldsymbol{CI_i}$ (gCO$_2$/kWh) & $\boldsymbol{B_i}$ (gCO$_2$) & $\boldsymbol{|D_i|}$ (images) \\
\hline
$c_1$ Sydney    & 2.1 & 420 & 800  & 12,000 \\
$c_2$ Singapore & 3.4 & 310 & 1500 & 8,000  \\
$c_3$ Berlin    & 4.8 & 180 & 2000 & 15,000 \\
$c_4$ Toronto   & 2.6 & 250 & 1200 & 10,000 \\
\hline
\end{tabular}
\end{table}

\section{Problem Formulation}
We consider a FLaaS system with one provider $P$ and a set of $N$ distributed consumers $\mathcal{C}=\{c_1,\ldots,c_N\}$, where each consumer $c_i$ holds a private local dataset $D_i$ of size $|D_i|$. The provider controls the overall service orchestration, including architecture selection, consumer participation, workload assignment, and aggregation strategy. We represent the federated service configuration as
\begin{equation}
\Omega=
\left(
\alpha,
\{S_t\}_{t=1}^{T},
\{E_i^t\}_{i\in S_t,t\le T},
Agg,
T
\right),
\end{equation}
where $(\alpha \in \mathcal{A})$ denotes the neural architecture selected from the search space $\mathcal{A}$, $(S_t \subseteq \mathcal{C})$ denotes the subset of participating consumers in round $t$, $(E_i^t \in \mathbb{Z}_{>0})$ denotes the number of local training epochs assigned to consumer $c_i$, and $\mathrm{Agg}(\cdot)$ denotes the aggregation operator. A consumer $c_i$ is characterised by heterogeneous sustainability and system profiles, including hardware efficiency $(\eta_i)$, communication profile $(\rho_i)$, time-varying carbon intensity $(CI_i(t))$, and carbon budget $(B_i)$. $(\eta_i)$ is measured in FLOPs/joule, $(\rho_i)$ captures communication overhead, $(CI_i(t))$ denotes the carbon intensity of the local power grid at time (t), and $(B_i)$ represents the total carbon budget available to the consumer during service execution. We assume that consumers follow the assigned training procedure, and the provider has access to the consumer's profile. Carbon budgets are treated as hard constraints. The carbon cost is modelled as the sum of computation and communication costs. The carbon cost incurred by consumer $c_i$ under architecture $\alpha$ and  $E_i^t$ at $t$ as
\vspace{-2mm}
\begin{equation}
Carbon_i^t(\alpha,E_i^t)
=
\frac{E_i^t\cdot Cost(\alpha)}{\eta_i}\cdot CI_i(t)
+
Comm(|\alpha|,\rho_i)\cdot CI_i(t)
\label{eq:carbon}
\end{equation}
where $Cost(\alpha)$ denotes the computation required for one local epoch and $|\alpha|$ denotes the model size, which determines the update transmission volume. The first term captures computation-related emissions, while the second captures communication-related emissions. Let $b_i(t)$ denote the remaining carbon budget of consumer $i$ at the start of round $t$. Its budget evolves as
\vspace{-2mm}
\begin{equation}
b_i(t+1)
=
b_i(t)
-
Carbon_i^t(\alpha,E_i^t)\cdot \mathbf{1}[i\in S_t], b_i(1)=B_i
\end{equation}
A consumer is feasible if its assigned workload does not exceed its remaining budget. The carbon-feasible participation of consumer $c_i$ under architecture $\alpha$ as
\vspace{-2mm}
\begin{equation}
\Gamma_i(\alpha)
=
\left\{
(t, E):
Carbon_i^t(\alpha,E)\le b_i(t)
\right\}
\end{equation}
The objective of the provider is to optimise the architecture and service orchestration to maximise the quality of the service model while ensuring carbon-feasible participation throughout training. The optimisation problem is defined as
\vspace{-2mm}
\begin{equation}
\begin{aligned}
\max_{\alpha\in\mathcal{A},\{S_t\},\{E_i^t\},Agg,T}
\quad & F(w_T(\alpha)) \\
\textrm{subject to}\quad
& F(w_T(\alpha))\ge \tau_{acc},\\
& T\le \tau_{deadline}, \\
& (t,E_i^t)\in\Gamma_i(\alpha),
\forall i\in S_t,\\
&|S_t|\ge K_{min} \forall t. \\
\end{aligned}
\label{eq:optimization}
\end{equation}
Here, $F(w_T(\alpha))$ denotes the final validation accuracy, $\tau_{acc}$ is the minimum accuracy, $\tau_{deadline}$ is the maximum allowable communication rounds, and $K_{min}$ is the minimum number of participating consumers for stable aggregation.

\section{Proposed SFLaaS-NAS Framework}
The proposed framework \textbf{SFLaaS-NAS}(\textit{\textbf{S}ustainable \textbf{F}ederated \textbf{L}earning \textbf{a}s \textbf{a} \textbf{S}ervice via \textbf{N}eural \textbf{A}rchitecture \textbf{S}earch}) consists of four components, as illustrated in Fig.~\ref{fig:framework}. First, a \textit{requirement-driven search space} transforms heterogeneous consumer sustainability constraints into a feasible architecture space. Second, a \textit{consumer-level carbon feasibility estimation} assesses whether individual consumers can implement a candidate architecture without exceeding their remaining carbon budgets. Third, a \textit{dynamic carbon-aware consumer
scheduling} selects sustainable participants and allocates local workloads based on carbon availability, data contribution, and computational efficiency. Finally, an \textit{evolutionary optimisation} strategy evaluates candidate architectures through communication rounds and iteratively
searches for architectures that maximise predictive performance, participation coverage, and carbon sustainability. We formulate SFLaaS-NAS using the following notations and definitions in Table~\ref{tab:notation}.
\vspace{-3mm}
\begin{figure}
\includegraphics[width=\textwidth]{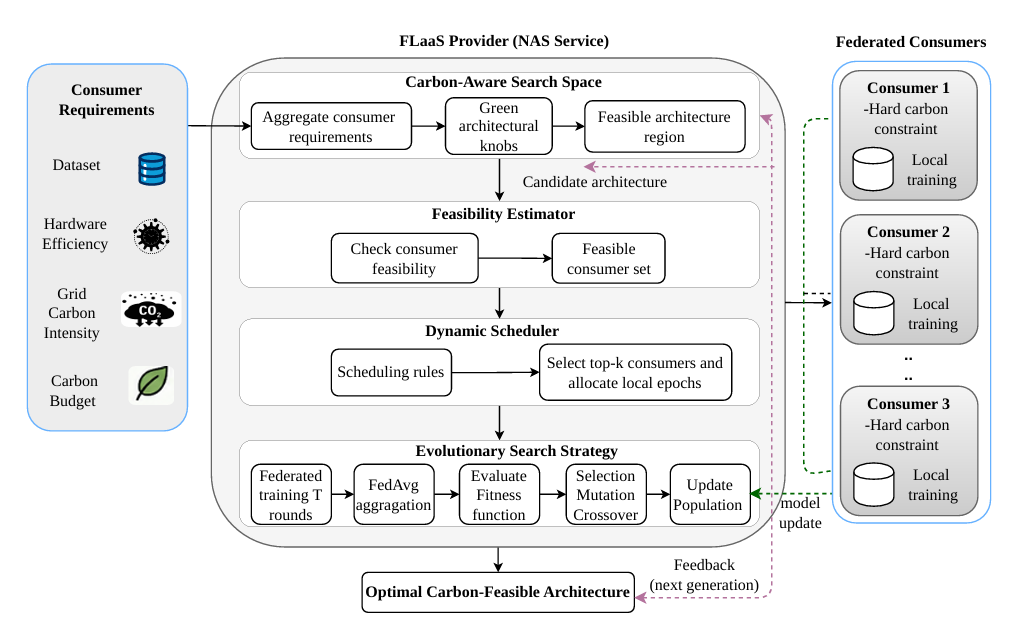}
\caption{An Overview of SFLaaS-NAS Framework} 
\label{fig:framework}
\vspace{-3mm}
\end{figure}

\begin{table}[t]
\centering
\caption{Notation and Description}
\label{tab:notation}
\footnotesize
\renewcommand{\arraystretch}{0.85}
\setlength{\tabcolsep}{4pt}
\begin{tabular}{ll}
\hline
\textbf{Notation} & \textbf{Description} \\
\hline
$\alpha \in \mathcal{A}$ & Candidate architecture \\
$R^{*}$ & Unified sustainability profile \\
$\mathcal{A}_{\mathrm{feasible}}$ & Feasible architecture set \\
$b_i(t)$ & Remaining carbon budget of consumer $i$ \\
$\mathrm{Carbon}_i^t(\alpha,E_i^t)$ & Consumer-side carbon cost \\
$\Gamma_i(\alpha)$ & Carbon-feasible participation region \\
$\phi_i(\alpha,t,E_i^t)$ & Feasibility indicator \\
$C_t^\alpha$ & Feasible consumer set \\
$S_t^\alpha$ & Scheduled consumer set \\
$Q_i^t(\alpha)$ & Scheduling score \\
$P_g$ & Population in generation $g$ \\
$\mathrm{Fitness}(\alpha)$ & Architecture fitness \\
$t$ & Communication round\\
\hline
\end{tabular}
\vspace{-2mm}
\end{table}

\subsection{Requirement-Driven Search Space}
We design a macro search space that explicitly exposes the architectural factors governing consumer-side carbon costs. To address the service architecture design challenge, consumer feasibility depends not only on performance but also on whether the selected architecture can be executed across heterogeneous constraints. Therefore, rather than searching over an arbitrary network configuration, we parameterise the architecture space using a set of carbon-aware green knobs whose impact on execution cost can be directly quantified. Each candidate architecture is represented by a set of green architectural knobs $\alpha =\{d,w,k,o,s,p\}\in\mathcal{A}$, where \textit{depth ($d$)} controls the number of computational blocks, \textit{width ($w$)} controls channels expansion,\textit{ kernel size ($k$)} controls the convolutional receptive field, \textit{operator type ($o$)} selects computational primitives such as standard and depthwise convolution with different energy efficiency, \textit{sparsity ($s$)} is the fraction of inactive weights, reducing active computations, \textit{precision ($p$)} defines numerical representation, affecting arithmetic energy, memory traffic, and communication cost. For $\alpha$, the provider estimates its computational complexity and parameter footprint $F(\alpha)=\text{FLOPs}(\alpha)$, and $P(\alpha)=\text{Params}(\alpha)$ to estimate both the local computation cost $\mathrm{Cost}(\alpha)$ and communication cost $\mathrm{Comm}(|\alpha|,\rho_i)$ required by the carbon model (eq. \ref{eq:carbon}), where $F(\alpha)$ represents the total floating-point operations required for one forward-backward pass, and $P(\alpha)$ denotes the total number of trainable
parameters. These determine the $\mathrm{Carbon}_i^t(\alpha, E_i^t)$ and communication cost in the carbon model. The search space $\mathcal {A} $ contains architectures with diverse predictive capacities; however, many candidate architectures may be inherently infeasible under heterogeneous sustainability constraints. To avoid unnecessary exploration of such architectures, we transform sustainability constraints into a constrained search region. Given the consumer set $\mathcal C=\{c_1,c_2,\ldots,c_N\}$,  the provider constructs a $R^*=(\eta^*,CI^*(t),B^*)$. The aggregated hardware efficiency is computed using the harmonic mean as $\eta^{*}=\left({1}/{N}\sum_{i=1}^{N}{1}/{\eta_i}\right)^{-1}$. The aggregated carbon intensity is computed as the dataset-weighted average as $CI^{*}(t)={\sum_{i=1}^{N}|D_i|CI_i(t)}/{\sum_{i=1}^{N}|D_i|}$. The carbon budget $B^{*}=\operatorname{Percentile} (\{B_i\}_{i=1}^{N},25)$. For $\alpha$, the provider
estimates the carbon cost under $R^*$ as $dCarbon(\alpha, R^*)$, which captures both computational and communication emissions under aggregated consumer conditions. Increasing network depth and width increases both
computational complexity and parameter count,
which, in turn, increases training and communication costs. Architectures exceeding the carbon budget are excluded, while sustainable architectures are retained for further search. The feasible architecture region defined in Eq.~\eqref{eq:feasible_search_space}, reducing the search space by eliminating architectures that violate sustainability constraints and focusing optimisation on sustainable architectures.
\vspace{-2mm}
\begin{equation}
\mathcal A_{feasible}
=
\left\{
\alpha\in\mathcal A
\;\middle|\;
dCarbon(\alpha,R^*)\le B^*
\right\}
\label{eq:feasible_search_space}
\end{equation}

\subsection{Consumer-Level Carbon Feasibility Estimation}
To maintain consumer feasibility, we evaluate whether a consumer can participate in $\alpha$, since architecture filtering does not guarantee that a consumer remains feasible throughout training. Consider a $\alpha\in\mathcal P^g$, where $g$ denotes the evolutionary generation. Suppose consumer $c_i$ is assigned $E_i^t$ local training epochs during $t$. The total computational workload of consumer $c_i$ is defined as
\vspace{-2mm}
\begin{equation}
W_i^t(\alpha)
=
F(\alpha)\cdot |D_i|\cdot E_i^t
\label{eq:workload}
\end{equation}
where $F(\alpha)$ denotes the per-sample computational complexity of $\alpha$, $|D_i|$ defines local dataset size, and $E_i^t$ define assigned local epochs. Given hardware efficiency $\eta_i$, the local training energy computed as $E_i^{train}(\alpha,t)={W_i^t(\alpha)}/{\eta_i}$. The communication energy is estimated as $E_i^{comm}(\alpha,t)=\rho_iP(\alpha)$, where $\rho_i$ denotes the communication energy coefficient and $P(\alpha)$ denotes the model parameter size. The total carbon cost during communication round $t$ is influenced by the time-varying carbon intensity $CI_i(t)$ of the consumer $c_i$, which is defined as
\vspace{-2mm}
\begin{equation}
\mathrm{Carbon}_i^t(\alpha,E_i^t)=
\left(
\frac{E_i^t \cdot \mathrm{Cost}(\alpha)}{\eta_i}
+
E_i^{\mathrm{comm}}(\alpha,\rho_i)
\right)
\frac{CI_i(t)}{3.6 \times 10^6}
\end{equation}

\begin{theorem}[Monotonic Carbon Cost and Consumer Feasibility]
For consumer $c_i$ and $\alpha$, $Carbon_i^t(\alpha,E_i^t)$ increases monotonically with workload $W_i^t(\alpha)$. Consumer $c_i$ remains feasible iff $(t,E_i^t)\in \Gamma_i(\alpha)\quad \Longleftrightarrow \quad Carbon_i^t(\alpha,E_i^t)\le b_i(t)$. If $(t,E_i^t)\notin \Gamma_i(\alpha)$, then $\phi_i(\alpha,t,E_i^t)=0$.
\end{theorem}
\begin{proof}
From Eq.~\eqref{eq:workload}, $W_i^t(\alpha)$ depends on
architecture complexity, local data volume and epochs. Since the $\eta_i$ is fixed for $c_i$ and $E_i^{\mathrm{comm}}(\alpha,t)\geq0$, the total carbon cost increases monotonically with workload. Since $CI_i(t)>0$, the carbon cost preserves monotonically with workload, and is defined as
\vspace{-2mm}
\begin{equation}
\frac{\partial Carbon_i^t(\alpha,E_i^t)}
{\partial W_i^t(\alpha)} \ge 0.
\end{equation}
By definition Eq.~\eqref{eq:carbon}, $c_i$ is feasible if and only if $(t,E_i^t)\in\Gamma_i(\alpha)$, which is equivalent to $Carbon_i^t(\alpha,E_i^t)\le b_i(t)$. Therefore, $\phi_i(\alpha,t,E_i^t)=1$ when the carbon constraint is satisfied, and $\phi_i(\alpha,t,E_i^t)=0$ otherwise.
\end{proof}
A consumer is feasible at $t$ if its assigned workload belongs to the feasible participation region. The \textit{feasibility indicator} is, $\phi_i(\alpha,t,E_i^t)= \mathbf{1}\left[(t, E_i^t)\in\Gamma_i(\alpha) \right]$. The feasible consumer set \textit{$C_t^\alpha$} in $\alpha$ during $t$ is defined as $\mathcal C_t^\alpha=\left\{c_i\in\mathcal C\;\middle|\;\phi_i(\alpha,t,E_i^t)=1\right\}$.

\subsection{Sustainable Consumer Scheduling}
To preserve participation under dynamic sustainability constraints, the provider must select sustainable consumers across communication rounds. Although all consumers in $\mathcal {C} _ {t} ^\alpha$ satisfy individual carbon-feasibility constraints, selecting all feasible consumers may rapidly deplete remaining carbon budgets, increase communication overhead, and amplify statistical heterogeneity. We implement a \textit{dynamic carbon-aware scheduler} that prioritises sustainability, statistical contribution, and carbon-efficient computation. For $\alpha$, the provider selects a scheduled consumer subset $S_t^\alpha\subseteq\mathcal C_t^\alpha$ based on three scheduling components.

\vspace{0.5em}
\noindent
\textbf{a) Carbon Availability Measurement (CAM):} CAM evaluates the remaining carbon budget, $b_i(t)$ of consumer $c_i$. Consumers with larger $b_i(t)$ are more likely to sustain participation without premature dropout, where the CAM is defined as
\vspace{-2mm}
\begin{equation}
\psi_i^t={b_i(t)}/{B_i}
\end{equation}

\vspace{0.5em}
\noindent
\textbf{b) Data Contribution Measurement (DCM):} DCM evaluates the contribution of consumer $c_i$ to the data. Consumers with higher values make a greater statistical contribution to global data coverage. The DCM score is defined as
\vspace{-2mm}
\begin{equation}
\delta_i={|D_i|}/{\sum_j |D_j|}
\end{equation}

\vspace{0.4em}
\noindent
\textbf{c) Carbon-Efficient Compute Measurement (CCM):} CCM evaluates the consumer $c_i$ under the computational efficiency of their current regional carbon conditions. Consumers with efficient hardware and low carbon intensity are preferred, which is defined as
\vspace{-2mm}
\begin{equation}
\xi_i^t={\eta_i}/{CI_i(t)}
\end{equation}
The unified scheduling score is computed as $Q_i^t(\alpha)=\lambda_1\psi_i^t+\lambda_2\delta_i+\lambda_3\xi_i^t$, where $\lambda_1$, $\lambda_2$, and $\lambda_3$ denotes scheduling weight. Consumers with larger remaining carbon budgets, greater data contribution, and more carbon-efficient hardware receive higher scheduling priority. To preserve aggregation stability, the provider selects the highest-ranked feasible participants s.t the minimum aggregation requirement, $S_t^\alpha= \operatorname{TopK}\left(\mathcal C_t^\alpha, Q_i^t(\alpha), K_t \right)$ where the round-wise consumer participation is denoted as $K_t= \min \left(|\mathcal C_t^\alpha|, K_{target}\right), K_t\ge K_{min}$. After that, local training workloads are allocated adaptively as shown in Eq.~\eqref{eq:adaptive_epochs}. If the assigned workload violates the consumer-side feasibility constraint, $(t, E_i^t)\notin\Gamma_i(\alpha)$. The provider iteratively reduces $E_i^t$ until feasibility is restored.
\vspace{-2mm}
\begin{equation}
E_i^t(\alpha)
=
\max
\left(
1,
\left\lfloor
E_{\max}
\frac{
b_i(t)
}{
B_i
}
\frac{
\eta_i
}{
\max_j\eta_j
}
\right\rfloor
\right)
\label{eq:adaptive_epochs}
\end{equation}

\begin{lemma}[Conditional Participation Stability]
Let us assume that every selected consumer satisfies $Carbon_i^t(\alpha, E_i^t)\le \lambda b_i(t),\quad0<\lambda<1$, then the number of scheduled consumers converges to a non-zero stable participation level as $\lim_{t\rightarrow\infty} |S_t^\alpha|=K^*>0$.
\end{lemma}

\begin{proof}
For each selected consumer, the remaining carbon budget evolves as
$b_i(t+1)=b_i(t)-Carbon_i^t(\alpha, E_i^t)$. Since $Carbon_i^t(\alpha,E_i^t)\le\lambda b_i(t)$. We obtain $b_i(t+1) \ge (1-\lambda)b_i(t)$. Recursively, $b_i(t)\ge (1-\lambda)^tb_i(0)$. Since $0<\lambda<1$, budget depletion occurs gradually rather than abruptly. The scheduler maintains a stable, non-empty, feasible subset, implying convergence to a stable participation level $K^*>0$.
\end{proof}

\subsection{Evolutionary Search Strategy under Sustainable Constraints}
We adopt an evolutionary search strategy over the feasible architecture space to optimise architecture performance, consumer feasibility, and participation coverage. This is suitable because the architecture variables are discrete, and the feasibility constraints are non-differentiable. Starting from an initial population sampled from $\mathcal A_{feasible}$, the provider evaluates candidate architecture, selects high-fitness architects as elite parents, and generates new offspring through mutation and crossover with green knobs. For $\alpha\in\mathcal P^g$, selected consumers perform local training, and aggregate updates using FedAvg \cite{nishio2019client}. We define the following
round-level utilities, such as predictive utility Eq.~\eqref{eq:accuracy_utility}, participation utility Eq.~\eqref{eq:participation_utility} and the carbon utility Eq.~\eqref{eq:carbon_utility} for optimasation.
\vspace{-2mm}
\begin{align}
U_{\mathrm{acc}}^t(\alpha)
&=
Acc^t(\alpha)
\label{eq:accuracy_utility}
\\
U_{\mathrm{part}}^t(\alpha)
&=
{|S_t^\alpha|}/{N}
\label{eq:participation_utility}
\\
U_{\mathrm{carbon}}^t(\alpha)
&=
\frac{1}{|S_t^\alpha|}
\sum_{i\in S_t^\alpha}
\frac{
b_i(t)
-
Carbon_i^t(\alpha,E_i^t)
}{
B_i+\epsilon
}
\label{eq:carbon_utility}
\end{align}
The round-level utilities are combined into a unified fitness function as
$U^t(\alpha)=\beta_1U_{\mathrm{acc}}^t(\alpha)+ \beta_2 U_{\mathrm{part}}^t(\alpha)+\beta_3 U_{\mathrm{carbon}}^t(\alpha)$, where $\beta_1,\beta_2,$ and $\beta_3$ denote non-negative utility weights satisfying $\beta_1+\beta_2+\beta_3=1$. The c\textit{umulative fitness} is represented as $Fitness(\alpha)=\frac{1}{T}\sum_{t=1}^{T} U^t(\alpha)$. Higher-fitness architectures are selected as \textit{elite parents}: $\mathcal E^g = TopM \left(\mathcal P^g, Fitness\right)$. The next generation is produced through mutation and crossover \cite{salmani2025systematic}, where mutation perturbs one or more green knobs, and crossover combines complementary design choices from elite parents as
\begin{equation}
    \mathcal P^{g+1}=\texttt{Mutate}\left(\mathcal E^g\right)\cup \texttt{Crossover}\left(\mathcal E^g\right) 
\end{equation}
After $G$ generations, the provider selects $\alpha^*=\arg\max_{\alpha\in\bigcup_g\mathcal P^g} Fitness(\alpha)$. The selected architecture $\alpha^*$ is deployed as the final FLaaS.

\begin{theorem}[Feasible Evolutionary Convergence]
Assume the feasible search space $\mathcal{A}_{\mathrm{feasible}}$ is finite and offspring violating $dCarbon(\alpha,R^{*}) \le B^{*}$ is discarded during offspring generation. If elite selection preserves the top-$M$ architectures ranked by fitness, every generation $\mathcal{P}^{g}$ remains feasible, and the search converges to a locally sustainable architecture $\lim_{g\rightarrow\infty} \max_{\alpha\in\mathcal{P}^{g}} Fitness(\alpha)=Fitness(\alpha^{*})$.
\end{theorem}

\begin{proof}
At generation $g$, architectures are sampled from the feasible search space. During offspring, mutation and crossover produce architectures violating the carbon-feasibility constraint $dCarbon(\alpha, R^{*}) \le B^{*}$. Such offspring
are discarded, ensuring $\mathcal{P}^{g+1} \subseteq \mathcal{A}_{\mathrm{feasible}}$. Elite selection preserves the top-$M$ architectures ranked by the fitness function. Let
$\alpha_i\in E^{g}$ denote an elite architecture and $\alpha_j\in\mathcal{P}^{g}\backslash E^{g}$ denote a non-elite architecture. Since elite architectures are preserved, the best fitness is non-decreasing: $\max_{\alpha\in\mathcal{P}^{g+1}}
Fitness(\alpha) \ge \max_{\alpha\in\mathcal{P}^{g}} Fitness(\alpha)$. Therefore, the fitness sequence is monotonically non-decreasing and bounded over the finite search space $\mathcal{A}_{\mathrm{feasible}}$, ensuring convergence to a locally sustainable architecture.
\end{proof}

\section{Experiment and Results}
A series of experiments is conducted to evaluate SFLaaS-NAS. We simulate a heterogeneous FLaaS with $N=20$ consumers over $T=20$ communication rounds, where ranges of parameters: $\eta_i \in [2.0, 5.0]$FLOPs/J, $\rho_i \in [0.5, 1.5]$J/parameter, $B_i\in[800, 2000] gCO_2$. The minimum aggregation requirement is fixed as $K_{\min}=5$ for stable aggregation. All datasets use a non-IID Dirichlet split ($\alpha=0.5$). CIFAR-10 \cite{cifar10} and ImageNet- 100\cite{imagenet_kaggle} represent vision tasks with increasing complexity, MedMNIST\cite{pathmnist} captures heterogeneous medical data across institutions, and WISDM \cite{wisdm2019} models sensor-driven applications in mobile environments. We design search space with green knobs, $\{d= 2,3,4,5\}$, $\{width=16,32,64\}$, $\{kernal size= 3,5\}$, operator type (standardand  depthwise convolution), $\{s= 0,0.25,0.5\}$, and $\{p=8,16,32\}$ bits, resulting in $|\mathcal{A}|=432$ candidate architectures. Experiments are implemented in Python on an AMD Ryzen 7 processor with AMD Radeon(TM) Graphics. The implementation and experimental code are available at: \footnote{\url{https://github.com/keyadata/SFLaaS-NAS}}

\subsection{Baselines}
We compare SFLaaS-NAS with several baselines that do not consider consumer-level sustainable constraints and evaluated using the same configurations.

\textbf{1. Accuracy-driven NAS}\cite{liu2019darts}: A conventional NAS approach that optimises only predictive accuracy on validation data.

\textbf{2. Efficiency-Aware NAS}\cite{cai2019proxylessnas}: This approach jointly considers predictive performance and model-complexity metrics, such as FLOPs and parameter count.

\textbf{3. Federated NAS}\cite{he2021fednas}: This strategy performs NAS within an FL setting and optimises model accuracy under non-IID data.

\textbf{4. Global Carbon-Aware NAS}\cite{zhou2020econas}: This approach selects an architecture based on overall energy or carbon efficiency.

\subsection{Evaluation Metrics}
We have employed the following metrics. \textit{Accuracy}, which measures the test accuracy of FLaaS after T rounds. \textit{Participation Rate} and \textit{Carbon Violation Rate}, which quantify the average fraction of consumers participating and proportion of selected consumers whose carbon cost exceeds their remaining carbon budget, respectively.
\vspace{-2mm}
\[
\text{Participation Rate}
=
\frac{1}{T}\sum_{t=1}^{T}\frac{|S_t|}{N}
\]

\[
\text{Carbon Violation Rate}
=
\frac{\sum_{t=1}^{T}\sum_{i\in S_t}
\mathbb{I}\!\left[\text{Carbon}_i^t(\alpha,E_i^t) > b_i(t)\right]}
{\sum_{t=1}^{T}|S_t|}
\]
We use separate validation and test partitions and evaluate across $5$ random seeds. Validation data select the fitness weights $\beta_1$--$\beta_3$ and scheduling weights $\lambda_1$--$\lambda_3$, while final performance is evaluated on unseen test partitions.
\subsection{Performance and Feasibility Analysis}
\textbf{Performance Evaluation under Carbon Constraints:}
We evaluate the effectiveness of SFLaaS-NAS in sustainable constraints across multiple datasets. The results in Table~\ref{tab:carbon_results} show that Accuracy-Driven NAS achieves strong performance but suffers from low participation (50-60\%) and high violation rates (25-32\%), limiting practicality under carbon constraints. Efficiency-Aware NAS improves participation (65-75\%) and reduces violations, but it lacks strict constraint enforcement, leading to non-negligible violations. Federated NAS balances accuracy and efficiency but shows limited participation (63-73\%) and notable violations (13-18\%), reducing suitability for carbon-intensive environments. Global Carbon-Aware NAS provides a better balance, achieving 70-80\% and lower violation rates without consumer heterogeneity. SFLaaS-NAS achieves high participation (around 90\%) and very low violation rates (1-2\%) across all datasets, demonstrating effective sustainable FLaaS orchestration under sustainability constraints.

\begin{table*}[t]
\caption{Performance Comparison under Carbon Constraints across Datasets}
\label{tab:carbon_results}
\resizebox{0.95\textwidth}{!}{%
\footnotesize
\renewcommand{\arraystretch}{0.85}
\setlength{\tabcolsep}{4pt}
\begin{tabular}{ccccc}
\hline
{\textbf{Dataset}} &
{\textbf{Method}} &
\textbf{Accuracy (\%)} &
\textbf{Participation (\%)} &
\textbf{Violation (\%)} \\
\hline

&
Accuracy-Driven NAS
& 85.8 $\pm$ 0.8
& 58.3 $\pm$ 2.6
& 26.7 $\pm$ 2.1 \\

&
Efficiency-Aware NAS
& 83.9 $\pm$ 0.9
& 74.6 $\pm$ 2.0
& 12.4 $\pm$ 1.7 \\

&
Federated NAS
& 85.1 $\pm$ 0.7
& 69.2 $\pm$ 2.4
& 15.8 $\pm$ 1.5 \\

{CIFAR-10} &
Global Carbon-Aware NAS
& 84.3 $\pm$ 0.6
& 78.5 $\pm$ 1.8
& 9.6 $\pm$ 1.2 \\

&
\textbf{SFLaaS-NAS}
& \textbf{86.4 $\pm$ 0.5}
& \textbf{94.8 $\pm$ 1.6}
& \textbf{1.2 $\pm$ 0.7} \\
\hline

&
Accuracy-Driven NAS
& 92.9 $\pm$ 0.6
& 61.2 $\pm$ 2.4
& 24.8 $\pm$ 1.8 \\

&
Efficiency-Aware NAS
& 91.3 $\pm$ 0.5
& 77.8 $\pm$ 1.9
& 10.9 $\pm$ 1.3 \\

&
Federated NAS
& 92.1 $\pm$ 0.6
& 73.5 $\pm$ 2.2
& 13.6 $\pm$ 1.4 \\

{WISDM} &
Global Carbon-Aware NAS
& 91.6 $\pm$ 0.4
& 81.3 $\pm$ 1.5
& 8.2 $\pm$ 1.0 \\

&
\textbf{SFLaaS-NAS}
& \textbf{93.7 $\pm$ 0.4}
& \textbf{96.1 $\pm$ 1.3}
& \textbf{0.9 $\pm$ 0.5} \\
\hline

&
Accuracy-Driven NAS
& 84.2 $\pm$ 0.7
& 56.8 $\pm$ 2.8
& 28.1 $\pm$ 2.2 \\

&
Efficiency-Aware NAS
& 82.7 $\pm$ 0.8
& 72.4 $\pm$ 2.1
& 13.2 $\pm$ 1.6 \\

&
Federated NAS
& 83.5 $\pm$ 0.6
& 68.7 $\pm$ 2.0
& 16.5 $\pm$ 1.5 \\

{MedMNIST} &
Global Carbon-Aware NAS
& 83.0 $\pm$ 0.5
& 76.9 $\pm$ 1.8
& 10.4 $\pm$ 1.1 \\

&
\textbf{SFLaaS-NAS}
& \textbf{83.6 $\pm$ 0.5}
& \textbf{80.0 $\pm$ 1.7}
& \textbf{1.0 $\pm$ 0.6} \\
\hline

&
Accuracy-Driven NAS
& 77.2 $\pm$ 1.1
& 49.5 $\pm$ 3.0
& 31.8 $\pm$ 2.4 \\

&
Efficiency-Aware NAS
& 75.1 $\pm$ 0.9
& 66.2 $\pm$ 2.5
& 16.1 $\pm$ 1.9 \\

&
Federated NAS
& 76.5 $\pm$ 0.8
& 63.8 $\pm$ 2.7
& 18.4 $\pm$ 1.8 \\

{ImageNet} &
Global Carbon-Aware NAS
& 75.8 $\pm$ 0.7
& 72.1 $\pm$ 2.1
& 12.3 $\pm$ 1.5 \\

&
\textbf{SFLaaS-NAS}
& \textbf{78.1 $\pm$ 0.6}
& \textbf{91.2 $\pm$ 1.9}
& \textbf{2.5 $\pm$ 0.8} \\
\hline
\end{tabular}
}
\end{table*}

\textbf{Feasibility Filtering Analysis:}
We evaluate the contribution of the feasibility filtering mechanism using CIFAR-10 by removing the feasibility-aware filtering condition in Eq.~\eqref{eq:feasible_search_space}, allowing search over the full architecture space. The results in Table~\ref{tab:feasibility_ablation} show that removing feasibility filtering reduces participation from 94.8\% to 82.1\% while increasing the carbon violation rate from 1.2\% to 11.4\%. This occurs because higher-cost architectures become infeasible under a strict carbon budget. In contrast, feasibility-aware filtering improves participation stability and reduces carbon violations in FLaaS.
\begin{table}[t]
\centering
\caption{Impact of Feasibility-Aware Search Space Reduction}
\label{tab:feasibility_ablation}
\footnotesize
\begin{tabular}{lcc}
\hline
\textbf{Method} & \textbf{Participation} (\%) & \textbf{Violation} (\%) \\
\hline
Without Feasibility Filtering
& 82.1 $\pm$ 2.3
& 11.4 $\pm$ 1.4 \\
SFLaaS-NAS
& \textbf{94.8 $\pm$ 1.6}
& \textbf{1.2 $\pm$ 0.7} \\
\hline
\end{tabular}
\end{table}

\begin{figure*}[t]
\vspace{-2mm}
\centering
\includegraphics[width=\textwidth]{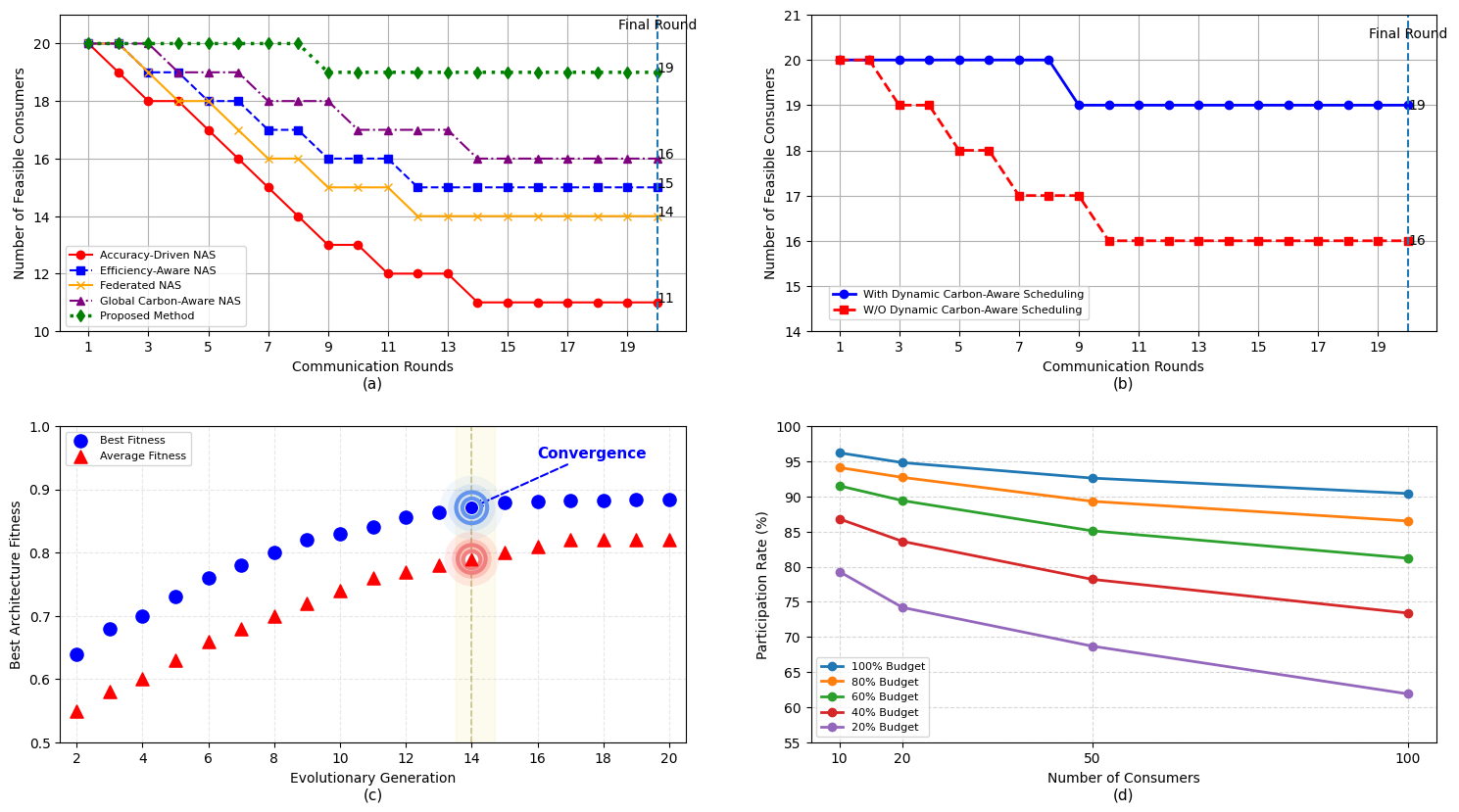}
\caption{(a) Consumer Feasibility Analysis, (b) Impact of Dynamic Scheduling, (c) Evolutionary Search Convergence, and (d) Scalability Analysis}
\label{fig:combined_results}
\vspace{-2mm}
\end{figure*}

\subsection{Consumer Feasibility and Scheduling Analysis}
\textbf{Consumer Feasibility Analysis:}
We further evaluate consumer feasibility and participation coverage using the CIFAR-10 dataset. Fig.~\ref{fig:combined_results}(a) shows the number of feasible consumers at the first and final round. All methods start with 20 consumers in Round 1, indicating that all consumers initially satisfy the carbon constraint. However, baseline methods show different levels of feasibility degradation by the final round. Accuracy-Driven NAS drops sharply from 20 to 11 consumers due to the selection of high-complexity architectures that quickly consume carbon budgets. Efficiency-Aware NAS reduces this degradation and retains 15 consumers, but does not enforce consumer-level carbon feasibility. Global Carbon-Aware NAS retains 16 consumers, but cannot handle heterogeneous consumer budgets. Interestingly, SFLaaS-NAS maintains 19 out of 20 feasible consumers in the final round, showing effective preservation of consumer participation by selecting architectures that satisfy consumer-level carbon constraints. SFLaaS-NAS not only improves final performance but also supports sustainable, stable FLaaS execution.

\textbf{Dymanic Scheduling Analysis:}
To evaluate adaptive consumer selection, we analyse the contribution of the dynamic carbon-aware scheduler in SFLaaS-NAS using CIFAR-10. We compare against a variant without dynamic carbon-aware scheduling, in which consumer selection does not adapt to round-wise carbon budgets. Fig.~\ref{fig:combined_results}(b) shows that, without dynamic scheduling, consumer feasibility degrades more rapidly, reducing consumers from 20 to 16 due to uncontrolled carbon budget consumption. In contrast, SFLaaS-NAS maintains stable participation (19/20) by selecting consumers based on their remaining carbon budgets. Removing dynamic scheduling also lowers participation 82.6\% vs. 94.8\%) and increases carbon violations (6.8\% vs. 1.2\%). These results show that dynamic scheduling plays a critical role in maintaining consumer feasibility and stable FL under carbon constraints.

\subsection{Search Convergence and Scalability Analysis}
\textbf{Evolutionary Search Convergence:}
We evaluate the convergence behaviour of the proposed evolutionary search on MedMNIST using $G=20$ generations, population size $|\mathcal P^g|=30$, elite size $M=10$, mutation probability $0.2$, and crossover probability $0.8$. For each generation, we record the best fitness architecture. The results in Fig.~\ref{fig:combined_results}(c) show that the best architecture fitness increases steadily from 0.64 to 0.88 by generation 14, after which the search gradually stabilises. Furthermore, no infeasible offspring are retained, indicating that all candidate architectures remain consumer-feasible across generations. Also, the monotonic improvement in fitness validates the effectiveness of elite preservation. These results confirm effective convergence toward stable carbon-feasible architectures.

\textbf{Scalability under Hard Carbon Constraints:}
To evaluate the scalability of SFLaaS-NAS on ImageNet-100, we increase the number of consumers from 10 to 100, while scaling each consumer's carbon budget $B_i$ to $100\%$, $80\%$, $60\%$, $40\%$, and $20\%$ of its original value. The results in Fig.~\ref{fig:combined_results}(d) show that participation gradually decreases as the number of consumers increases and carbon budgets become stricter due to greater consumer heterogeneity and reduced allowable training workloads. Nevertheless, SFLaaS-NAS maintains stable participation through carbon-feasible architecture selection and adaptive consumer scheduling. Even with $N=100$ and a $20\%$ carbon budget, SFLaaS-NAS sustains over $60\%$ participation, demonstrating strong scalability in large-scale FLaaS.
\begin{figure*}[t]
\vspace{-2mm}
\centering
\includegraphics[width=\textwidth]{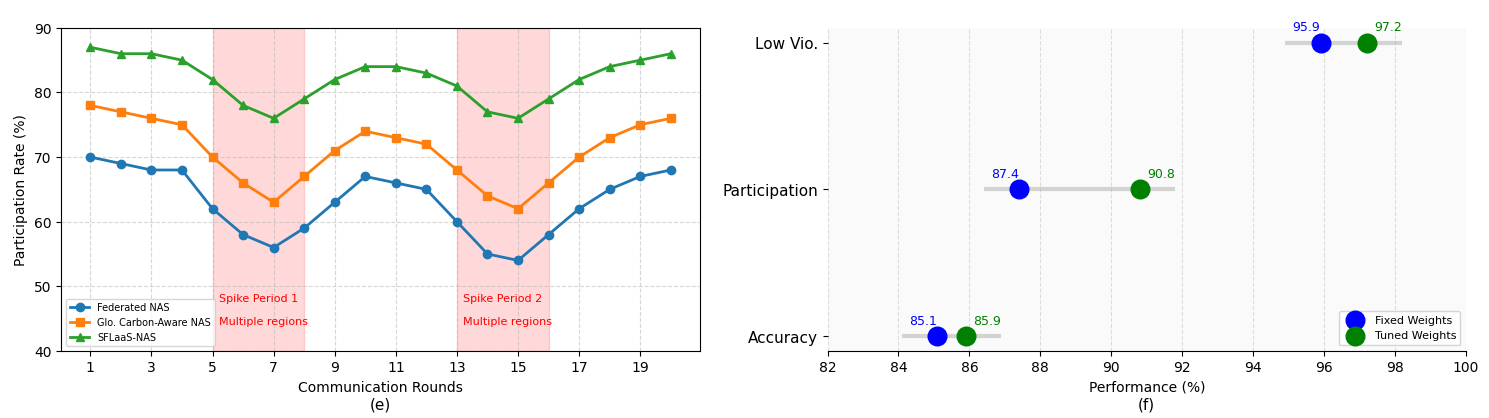}
\caption{(e) Robustness under Correlated Carbon-Intensity Spikes, and (f) Sensitivity Analysis under Fixed and Tuned Weight Configurations.}
\label{fig:e_f}
\vspace{-2mm}
\end{figure*}
\subsection{Robustness and Sensitivity Analysis}
\textbf{Robustness under Correlated Carbon-Intensity Spikes:}
We evaluate the robustness of SFLaaS-NAS under correlated regional carbon-intensity spikes using hourly Electricity Maps \cite{electricitymaps} based on geographic locations. We model synchronised high-carbon periods to evaluate participation stability under correlated regional carbon-intensity spikes, as shown in Fig.~\ref{fig:e_f}(e). During spike periods, Federated NAS drop to around 55\% due to computationally intensive architectures violating carbon constraints under elevated regional carbon intensity. Global Carbon-Aware NAS improves participation to 62-66\% during spike intervals by reducing average carbon consumption. In contrast, SFLaaS-NAS maintains participation above 75\% across most spike periods through carbon-feasible architecture selection and adaptive consumer scheduling, demonstrating stronger robustness under correlated carbon-constrained conditions. The results support Lemma~1, as stable participation and low violation rates indicate bounded consumer-side carbon expenditure.

\textbf{Sensitivity to Weight Selection:}
We evaluate whether the performance gains of SFLaaS-NAS arise from the proposed methodology or hyperparameter tuning. Specifically, we compare the tuned configuration against a fixed-weight setting with uniform weights, where $\beta=(1/3,1/3,1/3)$ and $\lambda=(1/3,1/3,1/3)$. The results in Fig.~\ref{fig:e_f}(f) show that SFLaaS-NAS maintains strong performance under fixed-weight, achieving 85.1\% accuracy, 87.4\% participation, and only 4.1\% carbon-violation rate. Validation-based tuning further improves participation to 90.8\% while reducing violations to 2.8\%. However, the small performance gap between the fixed and tuned configurations indicates that improvements come from the proposed feasible search space, carbon estimation, and dynamic scheduling rather than from extensive hyperparameter tuning alone.

\section{Discussion}
Experimental results show that SFLaaS-NAS improves consumer participation while reducing carbon-constraint violations by integrating consumer-level feasibility estimation, dynamic scheduling, and evolutionary NAS. However, one key limitation of SFLaaS-NAS is that the proposed evolutionary NAS effectively discovers carbon-feasible architectures; it introduces additional computational overhead during the search phase. We focused on the feasibility of sustainable architecture rather than minimising NAS search cost. Investigating lightweight or low-cost search strategies could improve the practicality of large-scale FLaaS deployment. Moreover, the current framework assumes a single FLaaS provider. In large-scale practical deployments, multiple providers operate with different infrastructure capabilities and sustainability goals. Such scenarios introduce additional challenges in coordination, resource allocation, fairness, and carbon-aware scheduling across providers. Extending the framework toward distributed multi-provider FLaaS remains an important direction for future work.
\section{Conclusion}
The proposed SFLaaS-NAS framework enables sustainable architecture design for FL under consumer-level carbon constraints. By integrating a feasibility-aware search space, consumer-level carbon estimation, and dynamic carbon-aware scheduling, the framework selects architectures that maintain high participation while satisfying carbon budgets. Experimental results show that SFLaaS-NAS sustains over 90\% participation with near-zero violations across multiple datasets, demonstrating effectiveness in heterogeneous environments. Future work will investigate lightweight sustainable NAS strategies and distributed multi-provider FLaaS environments to further improve scalability and efficiency.

%
%

\footnotesize
\begin{thebibliography}{8}
\bibitem{wen2023survey} Wen, J. et al.: A survey on federated learning: challenges and applications. Int. J. Mach. Learn. Cybern. 14(2), 513--535 (2023)

\bibitem{gao2024flaas} Gao, W. et al.: FLaaS for hierarchical edge networks with heterogeneous models. In: ICSOC. Springer (2024)

\bibitem{kourtellis2020flaas} Kourtellis, N. et al.: Flaas: Federated learning as a service. In: 1st Workshop on Dist.ML (2020).

\bibitem{albelaihi2022green} Albelaihi, R. et al.: Green federated learning via energy-aware client selection. In: IEEE GLOBECOM (2022).

\bibitem{salmani2025systematic} Salmani Pour Avval, S. et al.: Systematic review on neural architecture search. Artif. Intell. Rev. 58(3), 73 (2025)

\bibitem{li2021fedgreen} Li, P., Huang, X., Pan, M., Yu, R.: FedGreen: Federated learning with fine-grained gradient compression for green mobile edge computing. In: GLOBECOM (2021).

\bibitem{arputharaj2025green} Arputharaj, D. et al.: Green Federated Learning via Carbon-Aware Client and Time Slot Scheduling. In: MASCOTS (2025).

\bibitem{xu2021knas} Xu, J. et al.: KNAS: Green neural architecture search. In: ICML, PMLR (2021)

\bibitem{zhou2020econas} Zhou, D. et al.: 
EcoNAS: Finding proxies for economical neural architecture search. In: Proceedings of the IEEE/CVF (CVPR), pp. 11396--11404 (2020)

\bibitem{mazzocca2023truflaas} Mazzocca, C. et al.: TruFLaaS: Trustworthy federated learning as a service.  IEEE IoT J. 10(24), 21266--21281 (2023)

\bibitem{gao2024nebulafl} Gao, F. et al.: 
NebulaFL: Effective asynchronous federated learning for joint cloud computing. arXiv:2412.04868 (2024)

\bibitem{pathmnist} Kather, J. et al.: Predicting Survival from Colorectal Cancer Histology Slides Using DL: A Retrospective Multicenter Study. \textit{PLOS Medicine} (2019).

\bibitem{nishio2019client} Nishio, T., Yonetani, R.: Client selection for federated learning with heterogeneous resources in mobile edge. In:(ICC), pp. 1--7. IEEE (2019)

\bibitem{liu2019darts} Liu, H. et al.: Darts: Differentiable architecture search. In: ICLR (2019).

\bibitem{cai2019proxylessnas} Cai, H. et al.:  ProxylessNAS: Direct neural architecture search on target task and hardware. In: ICLR (2019)

\bibitem{cifar10} Krizhevsky, A.: Learning Multiple Layers of Features from Tiny Images. Technical Report, University of Toronto (2009).

\bibitem{wisdm2019} Weiss, G.: WISDM Smartphone and Smartwatch Activity and Biometrics Dataset. UCI ML Repository (2019). \doi{10.24432/C5HK59}

\bibitem{imagenet_kaggle} Howard, A. et al.: 
ImageNet object localisation challenge. 
Kaggle (2018). \url{https://kaggle.com/competitions/imagenet-object-localization-challenge}

\bibitem{lai2021oort} Lai, F. et al.: 
Oort: Efficient federated learning via guided participant selection. In: 15th USENIX Symposium on OSDI, pp. 19--35 (2021)

\bibitem{tan2019mnasnet} Tan, M. et al.: 
MnasNet: Platform-aware neural architecture search for mobile. In: Proceedings of the IEEE/CVF Conference on CVPR, pp. 2820--2828 (2019)

\bibitem{wu2019fbnet} Wu, B. et al.: 
FBNet: Hardware-aware efficient ConvNet design via differentiable neural architecture search. 
In: Proceedings of the IEEE/CVF (CVPR), (2019)

\bibitem{dou2023ea} Dou, S. et al.: 
EA-HAS-Bench: Energy-aware hyperparameter and architecture search benchmark. In: The Eleventh (ICLR) (2023)

\bibitem{anthony2020carbontracker} Anthony, L.F.W. et al.: CarbonTracker: Tracking and predicting the carbon footprint of training deep learning models. arXiv:2007.03051 (2020)

\bibitem{he2021fednas} He, C. et al.: 
FedNAS: Federated deep learning via neural architecture search (2021)

\bibitem{electricitymaps} Electricity Maps: Live carbon intensity of electricity consumption.https://app.electricitymaps.com/

\end{thebibliography}
%

\end{document}